\documentclass{article}
\usepackage{spconf}
\usepackage[utf8]{inputenc}
\usepackage{amsmath,amsthm,amsfonts,amssymb,bm,accents,dsfont}
\usepackage{graphicx,cite,enumerate}
\usepackage{url}
\graphicspath{{Images/}}

\usepackage{tikz}
\usepackage{pgfplots}

\usepackage{blindtext}

\newcommand{\executeiffilenewer}[3]{%
\ifnum\pdfstrcmp{\pdffilemoddate{#1}}%
{\pdffilemoddate{#2}}>0%
{\immediate\write18{#3}}\fi%
}

\usepackage[english]{babel}

\usepackage{algpseudocode,algorithm}
\algrenewcommand\algorithmicdo{}
\algrenewtext{EndFor}{\algorithmicend}
\algrenewtext{EndProcedure}{\algorithmicend}

\usepackage{color}

\newtheorem{theorem}{Theorem}
\newtheorem{proposition}{Proposition}
\newtheorem{corollary}{Corollary}

\theoremstyle{definition}
\newtheorem{definition}{Definition}

\newtheoremstyle{assume}
  {3pt}
  {3pt}
  {}
  {}
  {\bf}
  {}
  { }
  {\thmname{#1}.\thmnumber{#2}\thmnote{ \textnormal{(\textit{#3})}}}
\theoremstyle{assume}
\newtheorem{assumption}{A}

\DeclareMathOperator{\E}{\mathbb{E}}

\DeclareMathOperator*{\minimize}{minimize}

\DeclareMathOperator{\subjectto}{subject\ to}

\newcommand{\norm}[1]{\ensuremath{\left\| #1 \right\|}}
\newcommand{\abs}[1]{\ensuremath{{\left\vert #1 \right\vert}}}

\newcommand{\calD}{\ensuremath{\mathcal{D}}}

\newcommand{\calF}{\ensuremath{\mathcal{F}}}

\newcommand{\calH}{\ensuremath{\mathcal{H}}}

\newcommand{\calL}{\ensuremath{\mathcal{L}}}

\newcommand{\calO}{\ensuremath{\mathcal{O}}}
\newcommand{\calP}{\ensuremath{\mathcal{P}}}

\newcommand{\bx}{\ensuremath{\bm{x}}}

\newcommand{\btheta}{\ensuremath{\bm{\theta}}}
\newcommand{\blambda}{\ensuremath{\bm{\lambda}}}

\newcommand{\setR}{\ensuremath{\mathbb{R}}}

\def\st/{\textsuperscript{st}}
\def\nd/{\textsuperscript{nd}}
\def\rd/{\textsuperscript{rd}}
\def\th/{\textsuperscript{th}}

\def\nnil{\nil}
\newcounter{prob}
\newenvironment{prob}[1][\nil]{%
	\def\tmp{#1}
	\equation
	\ifx\tmp\nnil
		\refstepcounter{prob}
		\tag{P\Roman{prob}}
	\else
		\tag{\tmp}
	\fi
	\aligned%
}{%
	\endaligned\endequation%
}

\title{The empirical duality gap of constrained statistical learning}

\name{Luiz~F.~O.~Chamon, Santiago~Paternain, Miguel~Calvo-Fullana and Alejandro~Ribeiro}
\address{Electrical and Systems Engineering, University of Pennsylvania\\
{\small e-mail: \mbox{\texttt{\{luizf,spater,cfullana,aribeiro\}@seas.upenn.edu}}}}

\begin{document}
\ninept
\maketitle
\begin{abstract}

This paper is concerned with the study of constrained statistical learning problems, the unconstrained version of which are at the core of virtually all of modern information processing. Accounting for constraints, however, is paramount to incorporate prior knowledge and impose desired structural and statistical properties on the solutions. Still, solving constrained statistical problems remains challenging and guarantees scarce, leaving them to be tackled using regularized formulations. Though practical and effective, selecting regularization parameters so as to satisfy requirements is challenging, if at all possible, due to the lack of a straightforward relation between parameters and constraints. In this work, we propose to directly tackle the constrained statistical problem overcoming its infinite dimensionality, unknown distributions, and constraints by leveraging finite dimensional parameterizations, sample averages, and duality theory. Aside from making the problem tractable, these tools allow us to bound the empirical duality gap, i.e., the difference between our approximate tractable solutions and the actual solutions of the original statistical problem. We demonstrate the effectiveness and usefulness of this constrained formulation in a fair learning application.

\end{abstract}

\section{Introduction}
	\label{S:intro}

Statistical optimization problems are at the core of contemporary signal processing, machine learning, and statistical methods, from compressive sensing to image recognition to modern Bayesian methods~\cite{Eldar12c, Foucart13m, Rasmussen05g, mlbook}. Central to solving these problems is the concept of empirical risk minimization~(ERM), in which statistical quantities~(expectations) are replaced by their empirical counterparts~(sample averages) allowing the problem to be solved from data, without knowledge of the underlying distributions involved~\cite{mlbook, vapnik2013nature}. As such, ERM problems lie at the frontier of optimization and statistics. This approach is grounded on celebrated consistency results from learning theory that show that ERM solutions approach that of their statistical analog as the number of samples increases~\cite{vapnik2013nature, bousquet2003new}.

A limitation of the learning theory developed for ERMs is that it focuses on unconstrained optimization problems. However, modern information processing challenges require the ability to impose constraints so to incorporate prior knowledge via structural properties of the problem, such as smoothness or sparsity~\cite{wahba1990spline, berlinet2011reproducing, Eldar12c, Foucart13m}, or to deal with semi-supervised problems in which the data may be partially labeled~\cite{Cour, Yu2017, Nam}. Simultaneously, constraints can also describe desired properties of the solution, as is the case in risk-aware and fair learning~\cite{donini2018empirical, Bilal, Kalogerias2018b, Vitt2018, Zhou2018}.

Often, these constrained statistical problems are tackled by means of regularized formulations, i.e., by integrating a fixed constraint violation cost in the objective. While \emph{a priori} this is a reasonable approach, selecting this cost in practice may be challenging, requiring not only domain expertise but also mastering of the optimization algorithm of choice. This is due to the fact that there is typically no straightforward relation between the property we wish to embed in the solution~(i.e., the constraint we wish to impose) and the value of the regularization cost. This issue is only aggravated as the number of constraints grows. To this end, explicit constraints provide a clearer, more transparent approach to describing \emph{desiderata}.

These issues only exacerbate as the model complexity increases. Indeed, typical modern learning solutions involve non-convex, high dimensional parameterizations for which the relation between parameters and statistical properties is quite obscure. This is the case, for instance, of (convolutional/graph) neural networks~(NNs). Several heuristic constraints on the parameters of these architectures have been proposed, together with projection or conjugate gradient-type algorithms, to try and encode desirable features~\cite{Ravi, Karras}. In contrast, we wish to directly impose constraints on the statistical problem rather than attempt to induce structure on the parameterization in the hope that the solutions will then exhibit the desired properties.

In this work, we show that the solution of constrained, functional statistical learning problems can be approximated from data by using rich parameterizations, a result akin to the existing generalization bounds for ERM problems. We do so while directly account for the presence of constraints in the problem instead of relying on a hand-tuned regularization approach. Our main result is a bound on the empirical duality gap, i.e., the difference between the optimal value of the statistical problem and that of a particular finite dimensional, deterministic problem. This provides not only a way to perform near-optimal constrained learning, but also shows that solving constrained statistical problems is not harder than solving unconstrained ones.

We approach this result in two steps. First, we bound the \emph{parameterization gap}, i.e., the loss in optimality due to solving the problem using a finite dimensional parameterization rather than in the original functional space~(Section~\ref{S:parameterization}). We show that under mild conditions the solution of a specific saddle-point parameterized problem is close to the original functional one. This result depends on how good the parameterization is and not the specific parameterization, holding for a wide range of function classes including reproducing kernel Hilbert spaces~(RKHS) and NNs. We then proceed to bound the so-called \emph{empirical gap}, i.e., the error due to the use of samples instead of the unknown data distribution~(Section~\ref{S:erm}). The final bound on the empirical duality gap~(Theorem~\ref{T:main}) has a familiar form and depends not only on the number of samples, but also on the complexity of the statistical problem both in terms of the parametrization and how hard its constraints are to satisfy. We conclude with a numerical example in the context of fair learning to illustrate these results.

\section{Constrained statistical learning}
\label{S:csl}

Let~$\calD$ denote an \emph{unknown} joint probability distribution over pairs~$(\bx,y)$, with~$\bx \in \setR^d$ and~$y \in \setR$, and let~$\calF$ be a space of functions~$\phi: \mathbb{R}^d \to \mathbb{R}$, such as the space of continuous functions or~$L_2$. For convenience, we can interpret~$\bx$ as a feature vector or a system input, $y$ as a label or a measurement, and~$\phi$ as a classifier or estimator. We are interested in solving the constrained statistical learning problem
\begin{prob}[$\text{P-CSL}$]\label{P:csl}
	&&P^\star \triangleq \min_{\phi\in\calF}&
		&&\E_{(\bx,y) \sim \calD} \left[\ell_0(\phi(\bx),y)\right]
	\\
	&&\subjectto& &&\E_{(\bx,y) \sim \calD}\left[\ell_i(\phi(\bx),y)\right]
		\leq c_i \text{,}\ \ i = 1,\ldots,m
		\text{,}
\end{prob}
where~$\ell_0$ is a performance metric and~$\ell_i$ are functions that, together with the~$c_i$, encode the desired statistical properties of the solution of~\eqref{P:csl}. Observe that the unconstrained version of~\eqref{P:csl}, namely
\begin{prob}\label{P:sl}
	\minimize_{\phi\in\calF}&
		&&\mathbb{E}_{(\bx,y) \sim \calD} \left[\ell_0(\phi(\bx,y))\right]
		\text{,}
\end{prob}
has a long history in signal processing, statistics, and machine learning, being at the core of celebrated Bayesian estimators, such as LMS adaptive filters and Kalman filters, and virtually every modern learning algorithm~\cite{Kailath00l, Sayed08a, Rasmussen05g, mlbook}. Its success comes from the fact that, under mild conditions, the solution of the empirical version of~\eqref{P:sl}, i.e., one where the expectation is replaced by the sample mean over~$N$ realizations~$(\bx_n,y_n) \sim \calD$, converges rapidly to the solution of~\eqref{P:sl} as~$N$ grows~\cite[Section 3.4]{vapnik2013nature},\cite{bousquet2003new, mlbook}. For the sake of clarity, we omit the distribution~$\calD$ over which the expectations are taken in the remainder of this paper.

Two key challenges hinder the solution of~\eqref{P:csl} in general: (i)~it is an infinite dimensional problem, since it optimizes over a functional space and (ii)~we cannot evaluate its expectations since we do not have access to~$\calD$. The first issue can be handled by using a finite dimensional parameterization of~(a subset of) the space~$\calF$. Explicitly, we associate to each~$\btheta \in \calH \subseteq \setR^p$ a function~$f(\btheta,\cdot) \in \calF$, so that the statistical problem~\eqref{P:csl} is no longer over functions~$\phi \in \calF$, but over parameters~$\btheta \in \calH$. Here, $\calH$ is a convex set of admissible parameters. As for issue~(ii), it can be addressed as in the unconstrained case, i.e., by replacing expectations by their sample mean. However, since classical ERM theory gives guarantees for the unconstrained~\eqref{P:sl}~\cite[Section 3.4]{vapnik2013nature},\cite{bousquet2003new}, one often settles for a regularized version of~\eqref{P:csl} in which a fixed cost on the value of the constraints is included in the objective. Explicitly,
\begin{prob}\label{P:regularized}
	\minimize_{\btheta \in \calH}\ 
		\frac{1}{N} \sum_{n = 1}^N \left[ \ell_0(f(\btheta,\bx_n),y_n)
		+ \sum_{i = 1}^m w_i \ell_i(f(\btheta,\bx_n),y_n)
		\right]
		\text{,}
\end{prob}
where~$w_i \geq 0$ denote the weight of each constraint~(regularization parameter).

From a practical point-of-view, \eqref{P:regularized} has several advantages. In particular, we can directly use gradient descent methods to converge to a local minimum of its objective. Moreover, it is well-known that rich parameterizations such as NNs can be trained so as to make the objective of~\eqref{P:regularized} vanish for several commonly used~$\ell_0,\ell_i$~\cite{zhang2016understanding, arpit2017closer} or that all local minima are in fact global~\cite{soltanolkotabi2018theoretical, ge2017learning, brutzkus2017globally}. Nevertheless, there is no guarantee that the solution of~\eqref{P:regularized} is close to that of~\eqref{P:csl}. In fact, there may not even exist a set of weights~$w_i$ such that a solution of~\eqref{P:regularized} is \eqref{P:csl}-feasible. Even if it does, there is typically no straightforward way to find them. This is less of an issue if the specifications~$c_i$ and the weights~$w_i$ are obtained from data using, e.g., cross-validation. Still, the interpretation of the former is more transparent than the latter.

The goal of this work is to provide a systematic approach to obtain near-optimal solutions of~\eqref{P:csl} that converge to the optimum as the parameterization of~$\calF$ becomes richer and the number of samples from~$\calD$ grows. What is more, since this approach yields an optimization problem that, though distinct in nature, resembles~\eqref{P:regularized}, we effectively show that constrained statistical learning problems are in fact as easy~(or as hard) to solve as their unconstrained counterparts. To do so, we first analyze the \emph{parameterization gap}, showing that for rich parameterizations~(e.g., NNs), the optimality cost of parameterizing is quantifiably small~(Section~\ref{S:parameterization}). Then, we employ classical Vapnik-Chervonenkis~(VC) generalization results to bound the distance between the statistical and the empirical version of these problem~(Section~\ref{S:erm}). We then conclude by combining these bounds into a single empirical duality gap~(Section~\ref{S:empirical_duality}) that characterizes the difference between the solutions of the original constrained statistical problem and its tractable finite dimensional, deterministic version.

Before proceeding, we collect our assumptions on the functions from~\eqref{P:csl}. These are typical and met by a myriad of commonly used losses and cost functions~\cite{mlbook}.

\begin{assumption}\label{A:losses}

The functions~$\ell_0,\ell_i$ are non-negative, $B$-bounded, i.e., $[-B,B]$-valued, and~$L$-Lipschitz convex functions.

\end{assumption}

\section{The parameterization gap}
\label{S:parameterization}

On our path to obtain a practical solution for~\eqref{P:csl}, we begin by addressing issue~(i) from Section~\ref{S:csl}, namely, the fact that~\eqref{P:csl} is an infinite dimensional optimization problem. As is typical, we do so by leveraging a finite dimensional parameterization of~$\calF$ whose richness we characterize according to the following definition:

\begin{definition}\label{D:epsilon}

For~$\epsilon > 0$, $f$ is said to be an~$\epsilon$-parametrization of~$\calF$ if for each~$\phi \in \calF$ there exist~$\btheta \in \calH$ such that
\begin{equation}
	\E \left[ \abs{f(\btheta,\bx)-\phi(\bx)} \right] \leq \epsilon.
\end{equation}

\end{definition}

\noindent Note that Definition~\ref{D:epsilon} requires the set~$\calP = \{f(\btheta,\cdot) \mid \btheta \in \calH\}$ to be an $\epsilon$ cover of~$\calF$ in the total variation norm induced by the distribution of the data. This is considerably milder than the typical uniform approximation properties enjoyed by parameterization such as RKHSs or NNs~\cite{cybenko1989approximation, hornik1991approximation, hornik1989multilayer, berlinet2011reproducing}.

Given an~$\epsilon$-parameterization, we propose to tackle~\eqref{P:csl} through the saddle-point problem
\begin{prob}[$\widehat{\text{DI}}_{\epsilon}$]\label{P:dual_param}
	D^\star_\epsilon = \max_{\blambda \in \setR^m_+}\ \min_{\btheta\in\calH}\ 
		\calL\left( f(\btheta,\cdot),\blambda \right),
\end{prob}
where~$\calL$ is the Lagrangian of~\eqref{P:csl} defined as
\begin{equation}\label{E:lagrangian}
	\calL(\phi, \blambda) \triangleq \E \left[\ell_0(\phi(\bx),y)\right]
		+ \sum_{i = 1}^m \lambda_i \left(\E \left[\ell_i(\phi(\bx),y)\right]- c_i \right)
		\text{,}
\end{equation}
where $\lambda\in\mathbb{R}_+^m$ is a vector that collects the dual variables~$\lambda_i$.
Observe that~\eqref{P:dual_param} is nothing more than the dual problem of~\eqref{P:csl} solved over the parameterized function class~$\calP$. The parameterization gap compares the value of the original problem~\eqref{P:csl} and its parameterized dual version~\eqref{P:dual_param}, explicitly~$D^\star_\epsilon - P^\star$. If it is small, then the value of~\eqref{P:dual_param} is similar to that of~\eqref{P:csl}. As we show in the sequel, this also implies that solutions of~\eqref{P:dual_param} meet the constraints in~\eqref{P:csl}. Naturally, this can only occur if~$\calP$ has at least one function that is \eqref{P:csl}-feasible. In fact, in order to bound the parameterization gap, we need a more stringent condition:

\begin{assumption}\label{A:slater}

There exists a~$\btheta^\dagger \in \calH$ such that~$\E\left[\ell_i(f(\btheta^\dagger,\bx),y)\right] < c_i - L \epsilon $, for~$i = 1,\ldots,m$.

\end{assumption}

We can can then state the main result of this section:
\begin{proposition}\label{T:param}
Under assumptions~\ref{A:losses} and~\ref{A:slater}, if~$f$ is an~$\epsilon$-parameterization of~$\calF$, then
\begin{equation}\label{E:bound_param}
	P^\star \leq D^\star_{\epsilon} \leq P^\star +
		\left( 1 + \norm{\blambda_p^\star}_1 \right) L \epsilon
		\text{,}
\end{equation}
for~$P^\star$ and~$D^\star_\epsilon$ defined as in~\eqref{P:csl} and~\eqref{P:dual_param} respectively and where~$\blambda_p^\star$ are the dual variables of~\eqref{P:csl} with the constraints tightened to~$c_i - L\epsilon$.
 
\end{proposition}

\begin{proof}
See appendix~\ref{X:param}.
\end{proof}

Proposition~\ref{T:param} shows that for fine enough parameterizations of~$\calF$, solving the finite dimensional~\eqref{P:dual_param} yields almost the same solution as solving the variational~\eqref{P:csl}. How fine the parameterization must depends not only on the smoothness of the losses~$\ell_0,\ell_i$, but also on how hard to satisfy the constraints is. Indeed, optimal dual variables have a well-known sensitivity interpretation for convex problems: they describe how much the objective changes when the constraints are modified. Hence, if the constraints of the original statistical problem~\eqref{P:csl} are difficult to satisfy, then~\eqref{P:dual_param} will not be a good approximation unless the parameterization is very rich. On the other hand, if the constraints are easily satisfiable, i.e., practically inconsequential, then the only error we incur is that of approximating the optimal solution~$\phi^\star$ of~\eqref{P:csl}~(explicitly, $L\epsilon$).

An important corollary of Proposition~\ref{T:param} is that any solution obtained using~\eqref{P:dual_param} is~\eqref{P:csl}-feasible.

\begin{corollary}\label{T:param_feas}

Under assumptions~\ref{A:losses} and~\ref{A:slater}, the function~$f(\btheta^\star,\cdot)$ is a feasible solution of~\eqref{P:csl}, where~$\btheta^\star$ are the parameters that achieve~$D_\epsilon^\star$ in~\eqref{P:dual_param}.

\end{corollary}

\begin{proof}
See appendix~\ref{X:param_feas}.
\end{proof}

Despite being finite dimensional, \eqref{P:dual_param} remains a statistical problem. Hence, though Proposition~\ref{T:param} establishes that its solutions are not only~\eqref{P:csl}-feasible but also near-optimal, the issue remains of how to solve it without explicit knowledge of the distribution~$\calD$. Observe, however, that the objective of~\eqref{P:dual_param} now involves an unconstrained statistical minmization problem. In other words, we have done most of the heavy lifting already and can now rely on the generalization bounds from statistical learning to obtain a guarantee on our solution. This is the goal of the next section.

\section{The empirical gap}
\label{S:erm}

In order to solve~\eqref{P:dual_param}, we need to evaluate the Lagrangian in~\eqref{E:lagrangian}, which cannot be done without knowledge of the joint distribution~$\calD$. In practice, this issue is overcome by using realizations~$(\bx_n,y_n) \sim \calD$ to approximate the expectations by empirical averages. We therefore define the empirical Lagrangian as
\begin{equation}\label{E:empirical_lagrangian}
\begin{aligned}
	\calL_N(\btheta,\blambda) &= \frac{1}{N} \sum_{n=1}^N \bigg(
		\ell\left( f(\btheta,\bx_n),y_n \right)
	\\
	{}&+ \sum_{i=1}^m \lambda_i
			\left[ \ell_i\left( f(\btheta,\bx_n),y_n \right) - c_i \right] \bigg).
\end{aligned}
\end{equation}
and the empirical dual problem of~\eqref{P:csl} as
\begin{prob}[$\widehat{\text{DI}}_{\epsilon,N}$]\label{P:empirical_dual}
	D^\star_{\epsilon,N} = \max_{\blambda \in \setR^m_+}\ \min_{\btheta\in\calH}\ 
		\calL_N\left( f(\btheta,\cdot),\blambda \right).
\end{prob}
Note that the empirical Lagrangian has a form reminiscent of the regularized formulation in~\eqref{P:regularized}, which is appealing from a practical point-of-view since it means the minimization in~\eqref{P:empirical_dual} is not harder to solve than the regularized problem. Moreover, the dual function
\begin{equation}\label{E:empirical_d}
	d_N(\blambda) = \min_{\btheta \in \calH}\ \calL_N\left( f(\btheta,\cdot),\blambda \right)
\end{equation}
is concave by definition, since it is the pointwise minimum of a set of affine functions~\cite{Boyd04c}. After evaluating~\eqref{E:empirical_d}, the maximization in~\eqref{P:empirical_dual} can therefore be solved efficiently. Thus, when the parameterization~$f$ leads to an empirical Lagrangian that is convex in~$\btheta$, \eqref{E:empirical_d} can be solved exactly and efficiently. Even in the non-convex case, obtaining a good solution of~\eqref{E:empirical_d} is still be possible. This is the case, for instance, of NNs, since they can typically be trained so as to make the empirical Lagrangian vanish\cite{zhang2016understanding, arpit2017closer, soltanolkotabi2018theoretical, ge2017learning, brutzkus2017globally}.

Fundamental to our guarantees is the fact the empirical Lagrangian is uniformly close to its statistical version~\eqref{E:lagrangian}. This fact is expressed in the following theorem that bounds the gap between the optimal value of~\eqref{P:empirical_dual} and~\eqref{P:dual_param}:

\begin{proposition}\label{T:empirical}

Let~$d_{VC}$ denote the VC dimension of the hypothesis class~$\calP$. Then, under assumption~\ref{A:losses}, it holds with probability~$1-\delta$ that
\begin{equation}\label{E:empiricalGap}
	\abs{D^\star_{\epsilon} - D^\star_{\epsilon,N}} \leq V_N \triangleq 2 B
		\sqrt{\frac{1}{N} \left[ 1 + \log\left( \frac{4 (2N)^{d_{VC}}}{\delta} \right) \right]}
		\text{.}
\end{equation}
\end{proposition}

\begin{proof}
See appendix~\ref{X:empirical}.
\end{proof}

Before proceeding, it is worth noting that~\eqref{P:empirical_dual} is the dual problem associated with
\begin{prob}\label{P:empirical_constrained}
	\minimize_{\btheta \in \calH}& &&\frac{1}{N}\sum_{n=1}^N\ell(f(\boldsymbol{\theta},\bx_n))
	\\
	\subjectto& &&\frac{1}{N}\sum_{n=1}^N \ell_i\left( f(\btheta,\bx_n),y_n \right)
		\leq c_i \text{, } i = 1,\ldots,m.
\end{prob}
Yet, the relation between the parameterized, empirical problem~\eqref{P:empirical_constrained} and the statistical~\eqref{P:csl} stops at their form. Indeed, there is no guarantee that~\eqref{P:empirical_constrained} approximates~\eqref{P:csl}, since it is in general a non-convex optimization problem for which strong duality need not hold. The lack of strong duality also implies that~\eqref{P:empirical_constrained} cannot be solved through its dual~\cite{Boyd04c} and since projecting onto its non-convex constraints may not be straightforward, solving it directly could be quite challenging. Thankfully, we need not concern ourselves with~\eqref{P:empirical_constrained} since its dual problem~\eqref{P:empirical_dual} is close to~\eqref{P:csl}, as we show in the next section.

\section{The empirical duality gap}
\label{S:empirical_duality}

Combining Propositions~\ref{T:param} and~\ref{T:empirical}, we obtain the main result of this paper by using the triangle inequality:

\begin{theorem}\label{T:main}
For~$\epsilon > 0$, let~$f$ be an~$\epsilon$-parameterization of~$\calF$ as per Definition~\ref{D:epsilon} and consider the finite dimensional empirical dual~\eqref{P:empirical_dual} of~\eqref{P:csl}. Then, under assumptions~\ref{A:losses} and~\ref{A:slater}, it holds with probability~$1-\delta$ that
\begin{equation}\label{E:empirical_duality_gap}
	\abs{D^\star_{\epsilon,N} - P^\star} \leq
		\left(1 + \norm{\blambda_\epsilon^\star}_1 \right) L \epsilon + V_N,
\end{equation}
where~$P^\star$ and~$D^\star_{\epsilon,N}$ are the values of~\eqref{P:csl} and~\eqref{P:empirical_dual} respectively, $\blambda_p^\star$ are the dual variables that achieve~$D^\star_{\epsilon}$ in~\eqref{P:dual_param}, and~$V_N = \calO(\sqrt{\log(N)/N})$.
\end{theorem}

Hence, the finite dimensional empirical dual problem~\eqref{P:empirical_dual} can be used to obtain near-optimal solutions to the constrained statistical learning problem~\eqref{P:csl}. The quality of this approximation depends on the richness of the parameterization through~$\epsilon$ and the difficulty of it satisfying the constraints through~$\norm{\blambda_\epsilon^\star}_1$. The complexity of the parametrization also increases~$V_N$ through the VC dimension (see \eqref{E:empiricalGap}). Thus, there exist a synergy between the number of samples~$N$ and the parameterization. Indeed, for large sample sizes, the first term in~\eqref{E:empirical_duality_gap} dominates, motivating the use of richer representations. However, as the number of parameters of the representation grows, so does~$V_N$~\cite{mlbook}. In the low sample regimes, it may therefore be beneficial to use simpler parameterizations.

\begin{figure}[t]
	\centering
	\begin{tikzpicture}
	\pgfplotsset{grid style={dashed,gray!50}}
	\begin{axis}[
		tick scale binop=\times,
		minor tick num=1,
		xlabel={Epoch},
		ylabel={Test probabilities},
		ylabel near ticks,
		width=\columnwidth, 
		height=0.67\columnwidth, 
		xmin=0,xmax=300,ymin=0.80,ymax=1,
		grid=both,
		legend style={
			at={(0.98,0.26)}, 
			anchor=south east, 
			legend cell align=left
			}		
	]
	\addplot [color={rgb:red,55;green,126;blue,184},solid,thick,mark=none]
		table[x expr=\coordindex+1,y expr=\thisrow{ACC_log}, col sep=comma]{exp1.csv};
		\addlegendentry{Accuracy (Unconstrained)};

	\addplot [color={rgb:red,228;green,26;blue,28},solid,thick,mark=none]
		table[x expr=\coordindex+1,y expr=\thisrow{ACC_log}, col sep=comma]{exp3.csv};
		\addlegendentry{Accuracy (Constrained)};

	\addplot [color={rgb:red,55;green,126;blue,184},dashed,thick,mark=none]
		table[x expr=\coordindex+1,y expr=\thisrow{FAIR_log}, col sep=comma]{exp1.csv};
		\addlegendentry{Fairness (Unconstrained)};
		
	\addplot [color={rgb:red,228;green,26;blue,28},dashed,thick,mark=none]
		table[x expr=\coordindex+1,y expr=\thisrow{FAIR_log}, col sep=comma]{exp3.csv};
		\addlegendentry{Fairness (Constrained)};

\end{axis}
\end{tikzpicture}
	\vspace*{-8pt}
	\caption{Accuracy and fairness.}
	\label{F:performance}
\end{figure}
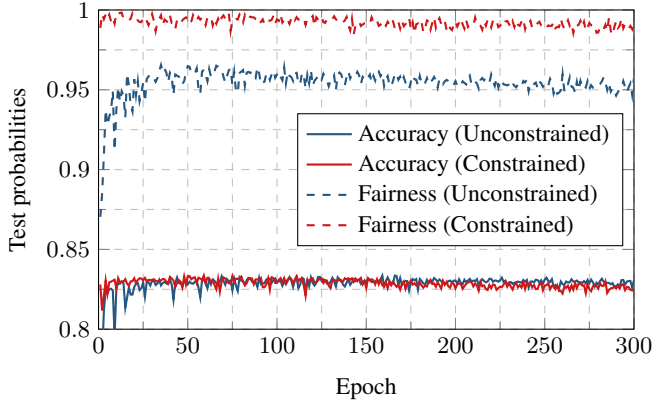

\section{Numerical Example}

In this section, we illustrate the use of the constrained statistical problem~\eqref{P:csl} and the empirical dual~\eqref{P:empirical_dual} in the context of fair learning~\cite{donini2018empirical,Bilal}. One of the ways that fairness is encoded in learning problems is by imposing that the final classifier be insensitive to a protected variable~$z$. If we consider that~$\phi(\bx,z)$ denotes the probability of the feature vector~$\bx$ with protected variable~$z$ belonging to the class~$y = 1$, then we would like~$\phi(\bx,z)$ to be similar regardless of the value of~$z$. In our numerical examples based on the Adult dataset~\cite{uci}, the goal is to predict whether an individual makes more than~US\$~$50.000$~(encoded as~$y = 1$) based on data such as age, weekly hours, marital status, and education. The protected variable in this case is gender. This problem can be stated as the constrained statistical problem
\begin{prob}\label{P:fair}
	\minimize_{\phi\in\calF}&
		&&-\E\left[y \log(\phi(\bx,z))\right]
	\\
	\subjectto& &&\E\left[D_\text{KL}(\phi(\bx,\text{Male}) \Vert \phi(\bx,\text{Female}))\right]
		\leq c
		\text{,}
\end{prob}
where~$D_\text{KL}(p \Vert q)$ denotes the Kullback-Liebler divergence between Bernoulli distributions with parameters~$p$ and~$q$ and~$c$ is the desired level of fairness. As a parameterization, we use a feedforward NN with~$256$ hidden nodes and sigmoidal activation function.

As a benchmark, we begin by training a classifier by ignoring the constraint in~\eqref{P:fair}, i.e., by using the unconstrained~\eqref{P:sl}. Results for accuracy~[defined as $\Pr(\phi(\bx_n,z)=y_i$] and fairness~[defined as $\Pr(\phi(\bx_n,\text{Male})=\phi(\bx_n,\text{Female}))$] on the test set are shown in the dashed lines of Figure~\ref{F:performance}. The test accuracy of the classifier is around~$83\%$ and it was insensitive to gender in approximately~$96\%$ of the test feature vectors. Both quantities ares estimated using empirical averages over the test set. Though this may be sufficiently high in many applications, impacting~$4\%$ of any population is often unacceptable. We therefore turn to~\eqref{P:fair} with~$c = 10^{-3}$.

Figure~\ref{F:lambda} shows the evolution of the dual variable~$\lambda$ over the training epochs~$k$. Notice that the dual variable converges to a value close to~$0.65$, which means that further tightening of relaxing the constraint would have little impact on the accuracy~(as measure by the cross-entropy loss~$\ell_0$). The resulting classifier is feasible for the constrained statistical problem as evidenced by Figure~\ref{F:slacks}, which shows the evolution of the slack~$\E\left[D_\text{KL}(f(\btheta_k,\bx,\text{Male}) \Vert f(\btheta_k,\bx,\text{Female}))\right] - c$. As shown in the solid lines of Figure~\ref{F:performance}, the estimator obtained using the empirical dual problem of~\eqref{P:fair} is insensitive to gender in~$99.7\%$ of the test samples at the cost of a loss in accuracy of less than~$0.1\%$.

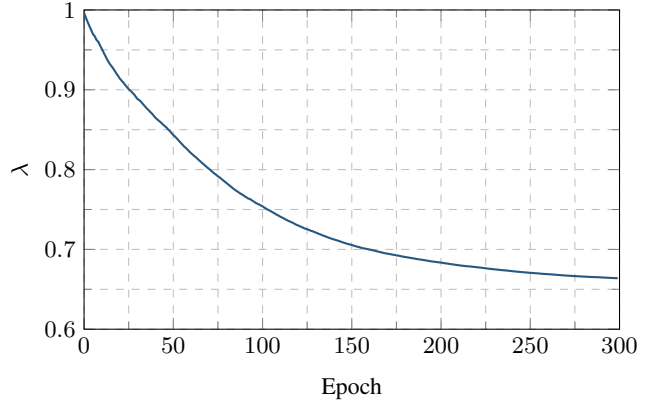
\begin{figure}[t]
	\centering
	\begin{tikzpicture}
	\pgfplotsset{grid style={dashed,gray!50}}
	\begin{axis}[
		tick scale binop=\times,
		minor tick num=1,
		xlabel={Epoch},
		ylabel={$\lambda$},
		ylabel near ticks,
		width=\columnwidth, 
		height=0.67\columnwidth, 
		xmin=0,xmax=300,ymin=0.6,ymax=1,
		grid=both,
	]
	
	\addplot [color={rgb:red,55;green,126;blue,184},solid,thick,mark=none]
		table[x expr=\coordindex,y expr=\thisrow{lambdas_log}, col sep=comma]{exp3.csv};
	
\end{axis}
\end{tikzpicture}
	\vspace*{-8pt}
	\caption{Dual variable relative to the fairness constraint.}
	\label{F:lambda}
\end{figure}

\begin{figure}[t]
	\centering
	\begin{tikzpicture}
	\pgfplotsset{grid style={dashed,gray!50}}
	\begin{axis}[
		tick scale binop=\times,
		minor tick num=1,
		xlabel={Epoch},
		ylabel={Constraint slack},
		ylabel near ticks,
		width=\columnwidth, 
		height=0.67\columnwidth, 
		xmin=0,xmax=300,ymin=-0.001,ymax=0,
		grid=both,
	]
	
	\addplot [color={rgb:red,55;green,126;blue,184},solid,thick,mark=none]
		table[x expr=\coordindex,y expr=\thisrow{PFEAS_log}, col sep=comma]{exp3.csv};
	
\end{axis}
\end{tikzpicture}
	\vspace*{-8pt}
	\caption{Fairness constraint slack.}
	\label{F:slacks}
\end{figure}
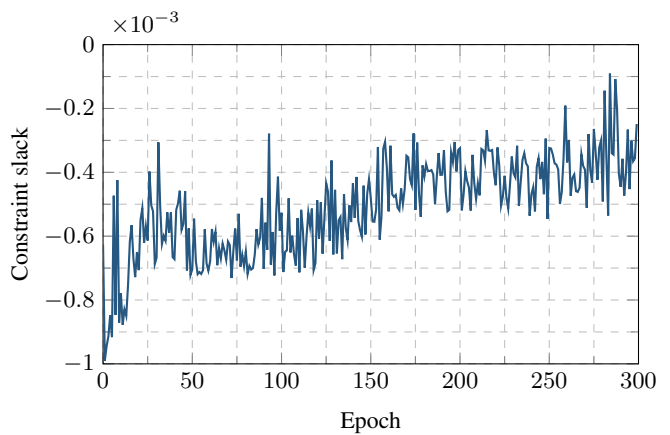

\section{Conclusion}

In this work, we have studied a constrained version of the celebrated statistical learning problem. Contrary to previous approaches based on regularization or heuristic structural properties of the parameterization, we directly tackle the constrained statistical problem. To do so, we proposed to overcome the infinite dimensionality, unknown distributions, and constrained nature of these problems by leveraging finite dimensional parameterizations, sample averages, and duality theory. Throughout this paper, we have shown that (i)~there is a bound on the optimality loss due to using parameterizations and that it can be made small by leveraging richer, higher-dimensional ones and (ii)~the gap due to using empirical average rather than statistical expectations decreases as the number of samples grows. Together, these results yield a bound on the empirical duality gap of constrained statistical learning problems, i.e., on the gap between the original constrained statistical problem~(the one we want to solve) and an unconstrained empirical problem~(the one we can solve). Finally, we have illustrated in a numerical example the practicality of our approach.

\newpage

\appendix

\section{Proof of Proposition~1}
\label{X:param}

\begin{proof}
The proof follows by relating the parameterized dual problem~\eqref{P:dual_param} to a perturbed~(tightened) version of the original~\eqref{P:csl}. We can then leverage the strong duality of these optimization problems to obtain the bounds in~\eqref{E:bound_param}.

Formally, start by defining the dual problem of~\eqref{P:csl} as
\begin{prob}[$\text{DI}$]\label{P:dual_csl}
	D^\star = \max_{\blambda \in \setR^m_+}\ \min_{\phi \in \calF}\ %
		\calL\left( \phi,\blambda \right),
\end{prob}
using the Lagrangian~$\calL$ from~\eqref{E:lagrangian}. Then, notice that assumption~\ref{A:slater} implies that strong duality holds for~\eqref{P:csl}, i.e., the saddle-point relation
\begin{multline}\label{E:saddle_csl}
	\calL( \phi^\star, \blambda^\prime ) \leq
	\max_{\blambda \in \setR^m_+}\ \min_{\phi \in \calF}\ %
		\calL\left( \phi,\blambda \right) = D^\star = P^\star
	\\	
	{}= \min_{\phi \in \calF}\ \max_{\blambda \in \setR^m_+}\ %
		\calL\left( \phi,\blambda \right)
	\leq \calL( \phi^\prime, \blambda^\star )
\end{multline}
holds for all~$\phi^\prime \in \calF$ and~$\blambda^\prime \in \setR^m_+$, where~$\phi^\star$ is a solution of~\eqref{P:csl} and~$\blambda^\star$ is a solution of~\eqref{P:dual_csl}. Additionally, we have from~\eqref{P:dual_param} that
\begin{equation}
	D_\epsilon^\star \geq \min_{\btheta \in \calH} \calL\left( f(\btheta, \cdot), \blambda \right)
		\text{,} \quad \text{for all } \blambda \in \setR^m_+
		\text{.}
\end{equation}
Immediately, we obtain the lower bound in~\eqref{E:bound_param}. Explicitly,
\begin{equation}
	D_\epsilon^\star
		\geq \min_{\btheta \in \calH} \calL\left( f(\btheta, \cdot), \blambda^\star \right)
		\geq \min_{\phi \in \calF} \calL\left( \phi, \blambda^\star \right) = P^\star
		\text{,}
\end{equation}
where the second inequality comes from~$\{f(\btheta,\cdot) \mid \btheta \in \calH\} \subseteq \calF$.

To derive the upper bound, first add and subtract~$\min_{\phi \in \calF} \calL(\phi,\blambda)$ from~\eqref{P:dual_param} to get
\begin{align*}
	D^\star_\epsilon &= \max_{\blambda \in \setR^m_+}\ %
		\min_{\substack{\btheta \in \calH, \\ \phi \in \calF}}\ %
		\calL\left( \phi, \blambda \right) +
		\E \left[\ell_0(f(\btheta,\bx),y) - \ell_0(\phi(\bx),y)\right]
	\\
	{}&+ \sum_{i = 1}^m \lambda_i
			\E \left[\ell_i(f(\btheta,\bx),y) - \ell_i(\phi(\bx),y)\right]
		\text{.}
\end{align*}
Notice that the last two terms can be written as the inner product~$\sum_{i = 0}^m \lambda_i \E \left[\ell_i(f(\btheta,\bx),y) - \ell_i(\phi(\bx),y)\right]$ by letting~$\lambda_0 = 1$. Hence, using H\"{o}lder's inequality yields
\begin{equation}\label{E:bound_holder}
\begin{aligned}
	D^\star_\epsilon &\leq \max_{\blambda \in \setR^m_+}\ %
	\min_{\substack{\btheta \in \calH, \\ \phi \in \calF}}\ %
		\calL\left( \phi, \blambda \right)
	\\
	{}&+ (1+\norm{\blambda}_1)
	\max_{i = 0,\dots,m}
		\abs{ \E \left[\ell_i(f(\btheta,\bx),y) - \ell_i(\phi(\bx),y)\right] }
		\text{.}
\end{aligned}
\end{equation}
The second term in~\eqref{E:bound_holder} can be bounded uniformly using the fact that~$f$ is an~$\epsilon$-parameterization. Indeed, using the convexity of the max-norm yields
\begin{multline*}
	\max_{i = 0,\dots,m}
		\abs{ \E \left[\ell_i(f(\btheta,\bx),y) - \ell_i(\phi(\bx),y)\right] }
	\leq{}
	\\
	\E \left[\max_{i = 0,\dots,m} \abs{\ell_i(f(\btheta,\bx),y) - \ell_i(\phi(\bx),y)}\right]
		\text{,}
\end{multline*}
which from assumption~\ref{A:losses} implies that 
\begin{equation}\label{E:bound_lipschitz}
	\max_{i = 0,\dots,m}
		\abs{ \E \left[\ell_i(f(\btheta,\bx),y) - \ell_i(\phi(\bx),y)\right] }
	\leq L \E \left[\abs{f(\btheta,\bx) - \phi(\bx)}\right]
		\text{.}
\end{equation}
Since~$f$ is an $\epsilon$-parameterization of~$\calF$, minimizing~\eqref{E:bound_lipschitz} over~$\btheta$ yields
\begin{multline}\label{E:bound_epsilon}
	\min_{\btheta \in \calH}\ \max_{i = 0,\dots,m}
		\abs{ \E \left[\ell_i(f(\btheta,\bx),y) - \ell_i(\phi(\bx),y)\right] }
	\\
	{}\leq \min_{\btheta \in \calH} L \E \left[\abs{f(\btheta,\bx) - \phi(\bx)}\right]
	\leq L \epsilon
		\text{.}
\end{multline}
Substituting~\eqref{E:bound_epsilon} back into~\eqref{E:bound_holder}, we write
\begin{equation}\label{E:bound_pertubed}
	D^\star_\epsilon \leq \max_{\blambda \in \setR^m_+}\ %
	\min_{\phi \in \calF}\ %
		\calL\left( \phi, \blambda \right) + (1+\norm{\blambda}_1) L \epsilon \triangleq D^\star_p
		\text{.}
\end{equation}

The next step is to realize that the right-hand side of~\eqref{E:bound_pertubed}, namely~$D^\star_p$, is the dual problem of a perturbed version of~\eqref{P:csl}. We thus obtain another saddle-point relation similar to~\eqref{E:saddle_csl} that we can exploit to bound~$D^\star_p$, and therefore~$D^\star_\epsilon$, in terms of~$P^\star$. Formally, using that~$\lambda_i \geq 0$, 
$D^\star_p$ in~\eqref{E:bound_pertubed} can be written as
\begin{equation}\label{E:dual_perturbed}
\begin{aligned}
	D^\star_p &= \max_{\blambda \in \setR^m_+}\ \min_{\phi \in \calF}\ %
		\E \left[\ell_0(\phi(\bx),y)\right] + L \epsilon
	\\
	{}&+ \sum_{i = 1}^m \lambda_i \left(\E \left[\ell_i(\phi(\bx),y)\right] - c_i + L \epsilon \right)
		\text{,}
\end{aligned}
\end{equation}
where we recognize the dual problem of
\begin{prob}\label{P:csl_perturbed}
	&&P_p^\star \triangleq \min_{\phi\in\calF}&
		&&\E \left[\ell_0(\phi(\bx),y)\right] + L\epsilon
	\\
	&&\subjectto& &&\E \left[\ell_i(\phi(\bx),y)\right]
		\leq c_i - L\epsilon \text{,}\ \ i = 1,\ldots,m
		\text{.}
\end{prob}
Notice that~\eqref{P:csl_perturbed} is a convex optimization problem just as~\eqref{P:csl} and that due to assumption~\ref{A:slater}, it is also strongly convex. Immediately, we obtain the saddle-point relation
\begin{multline}\label{E:saddle_perturbed}
	\calL( \phi_p^\star, \blambda^\prime ) + (1+\norm{\blambda^\prime}_1) L \epsilon \leq
	\max_{\blambda \in \setR^m_+}\ \min_{\phi \in \calF}\ %
		\calL\left( \phi,\blambda \right) + (1+\norm{\blambda}_1) L \epsilon
	\\
	{}= D^\star_p = P^\star_p ={}
	\\	
	\min_{\phi \in \calF}\ \max_{\blambda \in \setR^m_+}\ %
		\calL\left( \phi,\blambda \right) + (1+\norm{\blambda}_1) L \epsilon
	\leq \calL( \phi^\prime, \blambda_p^\star ) + (1+\norm{\blambda_p^\star}_1) L \epsilon
		\text{,}
\end{multline}
which holds for all~$\phi^\prime \in \calF$ and~$\blambda^\prime \in \setR^m_+$, where~$\phi_p^\star$ is the solution of~\eqref{P:csl_perturbed} and~$\blambda_p^\star$ is the dual variable that achieves~$D^\star_p$ in~\eqref{E:dual_perturbed}.

Going back to~\eqref{E:bound_pertubed} we can now conclude the proof. First, use~\eqref{E:saddle_perturbed} to obtain
\begin{equation}
	D^\star_\epsilon \leq D^\star_p \leq \calL\left( \phi^\star, \blambda_p^\star \right)
		+ (1+\norm{\blambda_p^\star}_1) L \epsilon
		\text{,}
\end{equation}
where we used~$\phi^\prime = \phi^\star$, the solution of~\eqref{P:csl}. Now, using~\eqref{E:saddle_csl} on the Lagrangian term gives
\begin{equation}
	D^\star_\epsilon \leq \calL\left( \phi^\star, \blambda^\star \right)
		+ (1+\norm{\blambda_p^\star}_1) L \epsilon
	= P^\star + (1+\norm{\blambda_p^\star}_1) L \epsilon
		\text{,}
\end{equation}
which concludes the proof.
\end{proof}

\section{Proof of Corollary~1}
\label{X:param_feas}

\begin{proof}
Suppose that it is not the case. Then, there exists at least one~$i$ such that~$\E\left[ \ell_i(f(\btheta^\star,\bx),y) \right] - c_i > 0$. Since~$\blambda$ is unbounded above, we obtain that~$D^\star_\epsilon \to +\infty$. However, assumptions~\ref{A:losses} and~\ref{A:slater} imply that~$D^\star_\epsilon < +\infty$. Indeed, consider the dual function
\begin{equation}\label{E:param_d}
\begin{aligned}
	d(\blambda) &= \min_{\btheta \in \calH} \calL(f(\btheta,\cdot),\blambda)
	\\
	{}&= \min_{\btheta \in \calH} \E \left[\ell_0(f(\btheta,\bx),y)\right]
		+ \sum_{i = 1}^m \lambda_i \left(\E \left[\ell_i(f(\btheta,\bx),y)\right] - c_i \right)
		\text{.}
\end{aligned}
\end{equation}
Using the fact that~$\ell_0$ is $B$-bounded and the strictly feasible point~$\btheta^\dagger$ from assumption~\ref{A:slater}, $d(\blambda)$ is upper bounded by
\begin{align*}
	d(\blambda) &\leq \E \left[\ell_0(f(\btheta^\dagger,\bx),y)\right]
		+ \sum_{i = 1}^m \lambda_i
			\left(\E \left[\ell_i(f(\btheta^\dagger,\bx),y)\right] - c_i \right)
	\\
	{}&< B - (L\epsilon ) \sum_{i = 1}^m {\blambda}_i =  B - (L\epsilon ) \left\|{\blambda}\right\|_1 < \infty
		\text{,}
\end{align*}
where we once again used the fact that~$\lambda_i \geq 0$ to write~$\sum_i \lambda = \norm{\blambda}_1$. Hence, it must be that~$f(\btheta^\star,\cdot)$ is~\eqref{P:csl}-feasible.
\end{proof}

\section{Proof of Proposition~2}
\label{X:empirical}

\begin{proof}
Let~$(\btheta_{\epsilon}^\star, \blambda_\epsilon^\star)$ and~$(\btheta_{\epsilon,N}^\star, \blambda_{\epsilon,N}^\star)$ be variables that achieve~$D_{\epsilon}^\star$ in~\eqref{P:dual_param} and~$D_{\epsilon,N}^\star$ in~\eqref{P:empirical_dual} respectively. As such, both pair of variables satisfy the KKT conditions~\cite[Section 5.5.3]{Boyd04c}. In particular, it holds that
\begin{subequations}\label{E:comp_slack}
\begin{align}
	[\blambda_{\epsilon}^\star]_i \bigg(\E \left[\ell_i(f(\btheta_{\epsilon}^\star,\bx),y)\right]
		- c_i \bigg) &= 0
		\text{,}
		\label{E:comp_slack_param}
	\\
	[\blambda_{\epsilon,N}^\star]_i \left(\frac{1}{N} \sum_{n = 1}^N
		\ell_i(f(\btheta_{\epsilon}^\star,\bx_n),y_n)- c_i \right) &= 0
		\text{,}
		\label{E:comp_slack_emp}
\end{align}
\end{subequations}
known as \emph{complementary slackness} conditions. Immediately, \eqref{E:comp_slack} implies that both~\eqref{P:dual_param} and~\eqref{P:empirical_dual} reduce to
\begin{alignat*}{2}
	D_{\epsilon}^\star &= \E \left[\ell_0\left(f( \btheta_{\epsilon}^\star,\bx), y \right)\right]
		&&\triangleq F_0(\btheta_{\epsilon}^\star)
	\\
	D_{\epsilon,N}^\star &= \frac{1}{N} \sum_{n = 1}^N
			\ell_0\left(f( \btheta_{\epsilon,N}^\star, \bx_n), y_n \right)
		&&\triangleq \hat{F}_0(\btheta_{\epsilon,N}^\star)
\end{alignat*}
and the empirical gap we wish to bound can be written as
\begin{equation}\label{E:equivalentGap}
	\abs{D_{\epsilon}^\star - D_{\epsilon,N}^\star} = \abs{
		F_0(\btheta_{\epsilon}^\star) - \hat{F}_0(\btheta_{\epsilon,N}^\star)
	}
		\text{.}
\end{equation}

To proceed, use the optimality of~$\btheta_{\epsilon}^\star$ and~$\btheta_{\epsilon,N}^\star$ for~$F_0$ and~$\hat{F}_0$ respectively to write
\begin{equation*}
	F_0(\btheta_{\epsilon}^\star) - \hat{F}_0(\btheta_{\epsilon}^\star) \leq
		F_0(\btheta_{\epsilon}^\star) - \hat{F}_0(\btheta_{\epsilon,N}^\star) \leq
		F_0(\btheta_{\epsilon,N}^\star) - \hat{F}_0(\btheta_{\epsilon,N}^\star)
		\text{,}
\end{equation*}
which implies that~\eqref{E:equivalentGap} can be bounded as
\begin{multline}\label{E:gapBound}
	\abs{D_{\epsilon}^\star - D_{\epsilon,N}^\star} \leq
	\\
	\max \bigg\{
		\abs{F_0(\btheta_{\epsilon}^\star) - \hat{F}_0(\btheta_{\epsilon}^\star)},
		\abs{F_0(\btheta_{\epsilon,N}^\star) - \hat{F}_0(\btheta_{\epsilon,N}^\star)}
	\bigg\}
		\text{.}
\end{multline}
The proof now concludes by applying the classical VC generalization bound to~\eqref{E:gapBound}~\cite[Section 3.4]{vapnik2013nature}. Namely, for all~$\btheta$, it holds with probability~$1-\delta$ that
\begin{equation}\label{E:vcBound}
	\abs{F_0(\btheta) - \hat{F}_0(\btheta)} \leq V_N \triangleq 2 B
		\sqrt{\frac{1}{N} \left[ 1 + \log\left( \frac{4 (2N)^{d_{VC}}}{\delta} \right) \right]}
		\text{,}
\end{equation}
where~$d_{VC}$ is the VC dimension of~$\calP$. Since the bound is uniform over~$\btheta$, it holds also for the minimizers of~\eqref{P:dual_param} and~\eqref{P:empirical_dual}. Using~\eqref{E:vcBound} in~\eqref{E:gapBound} yields the desired bound in~\eqref{E:empiricalGap}.
\end{proof}


\bibliographystyle{IEEEbib}
\bibliography{IEEEabrv,gsp,sp,af,math,bayes,ml,library_fixed,kernel,mkl}

\end{document}